\documentclass[conference]{IEEEtran}
\IEEEoverridecommandlockouts
\usepackage{cite}
\usepackage{booktabs}
\usepackage{amsmath,amssymb,amsfonts}
\usepackage{amsthm,bm}
\usepackage[T1]{fontenc}
\usepackage{url}
\usepackage{graphicx}
\usepackage{textcomp}
\usepackage{xcolor}
\def\BibTeX{{\rm B\kern-.05em{\sc i\kern-.025em b}\kern-.08em
    T\kern-.1667em\lower.7ex\hbox{E}\kern-.125emX}}

\makeatletter
\newcommand{\linebreakand}{%
  \end{@IEEEauthorhalign}
  \hfill\mbox{}\par
  \mbox{}\hfill\begin{@IEEEauthorhalign}
}
\makeatother

\newtheorem{proposition}{Proposition}

\newcommand{\T}{^{\top}}

\newcommand{\R}{\mathbb{R}}

\newcommand{\traj}{\mathcal{T}}
\newcommand{\Hh}{\mathsf{H}}
\newcommand{\Aauth}{\mathsf{A}}
\newcommand{\Pphys}{\mathsf{P}}

\begin{document}

\title{From LLM Plans to Authorized Motion: Execution-Bound Authorization for Excavators%
\thanks{This research was supported by Mechanical Equipment Industry Technology Development Program through the Korea Planning \& Evaluation of Industrial Technology (KEIT) funded by the Korea government (Ministry of Trade, Industry and Energy (MOTIE)) (RS-2024-00442168).}
}

\author{
\IEEEauthorblockN{Mehdi Heydari Shahna\textsuperscript{*}}
\IEEEauthorblockA{
\textit{Faculty of Engineering}\\
\textit{and Natural Sciences}\\
\textit{Tampere University}\\
Tampere, Finland\\
mehdi.heydarishahna@tuni.fi\\
\textsuperscript{*}Corresponding author
}
\and
\IEEEauthorblockN{Seihun Kim}
\IEEEauthorblockA{
\textit{Department of Defense}\\
\textit{Convergence Technology}\\
\textit{Korea University}\\
Seoul, Republic of Korea
}
\and
\IEEEauthorblockN{Soyi Jung}
\IEEEauthorblockA{
\textit{Department of Electrical}\\
\textit{and Computer Engineering}\\
\textit{Ajou University}\\
Suwon, Republic of Korea
}
\linebreakand
\IEEEauthorblockN{Soohyun Park}
\IEEEauthorblockA{
\textit{Division of Computer Science}\\
\textit{Sookmyung Women's University}\\
Seoul, Republic of Korea
}
\and
\IEEEauthorblockN{Jouni Mattila}
\IEEEauthorblockA{
\textit{Faculty of Engineering}\\
\textit{and Natural Sciences}\\
\textit{Tampere University}\\
Tampere, Finland
}
\and
\IEEEauthorblockN{Joongheon Kim}
\IEEEauthorblockA{
\textit{Department of Electrical}\\
\textit{and Computer Engineering}\\
\textit{Korea University}\\
Seoul, Republic of Korea
}}

\maketitle

\begin{abstract}
An LLM-directed excavator can execute mechanically feasible motion whose destination or authority no longer matches the operator's approval. We propose execution-bound authorization for the approval--dispatch--execution interval. A one-use lease binds the action, generated motion, mission contract and admission tolerance; atomic dispatch rechecks current authority and measured joint drift. Periodic supervision detects revocation and requests a quintic stop with a carry-pitch endpoint, selected through sampled joint, stroke, flow and reference-acceleration checks. The framework connects exact motion identity to dispatch integrity and load-aware interruption. Evaluation combines a 672-case interface suite, four loaded digging cases in synchronous simulation, unloaded MuJoCo/Pinocchio transfers and two command-level LLM planners. The gateway admitted all 120 benign interface cases and produced zero unauthorized dispatches in the attack suite. Loaded dispatch-window tests produced zero unauthorized deposition. During simulated dump-phase revocation, maximum closed-site deposition decreased from 1.23 to 0.33\,m$^3$; loaded-slew stops used less than 75\% of nominal swing torque. These results demonstrate motion-bound authorization and quantify the retention and stopping behavior of supervised excavation skills.
\end{abstract}

\begin{figure}[!b]
\centering
\includegraphics[width=\columnwidth]{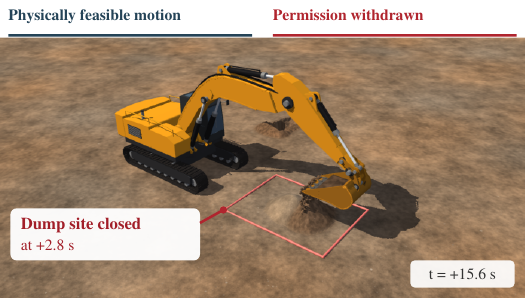}
\caption{The authorization challenge: a mechanically feasible dump continues after the operator closes its destination. This recorded simulation frame illustrates why current permission must govern an ongoing loaded skill as well as its initial approval.}
\label{fig:motivation}
\end{figure}

\section{Introduction}
Large language models (LLMs) are bringing natural-language interaction into everyday work. A study of ChatGPT usage documents workplace writing and decision support~\cite{chatterji2025chatgpt}, while field evidence demonstrates productivity gains from conversational assistance~\cite{brynjolfsson2025work}. Robotics research extends these capabilities to task planning, policy generation and spatial reasoning~\cite{liang2023codeaspolicies,singh2023progprompt,huang2023voxposer}. For industrial machines, language-based work orders offer a direct way to express goals and use site information, provided the resulting actions remain under the operator's authority.

The same contextual flexibility creates an attack surface. Instructions embedded in untrusted documents or messages can redirect LLM-integrated applications~\cite{greshake2023indirect}, and robot jailbreaks can elicit harmful physical actions~\cite{robey2025robopair}. An excavator can consequently perform a smooth, mechanically feasible dump at an unauthorized destination. Authority can also change after legitimate approval: Fig.~\ref{fig:motivation} shows a loaded skill continuing into a site that the operator has closed. Industrial adoption therefore requires checks that connect the meaning and authority of a work order to the motion and material transfer it produces.

Mission-bound authorization, contextual robot guardrails and agent privilege controls provide complementary defenses~\cite{chepuri2026authguard,ravichandran2026roboguard,yang2024safetychip,debenedetti2025camel,shi2026progent}. A long hydraulic skill additionally requires the approved action to remain bound to its exact motion across dispatch and execution. An individually valid substitute can otherwise replace it, an approval can become stale, or a loaded slew can outlive its permission. Interrupting the last case requires a checked stopping motion that accounts for the carried material; freezing the reference creates a substantial swing-drive transient in our simulations.

We propose \emph{execution-bound authorization} for this approval--dispatch--execution interval (Fig.~\ref{fig:teaser}) for the excavator application used in \cite{shahna2026act}, with three contributions: 1) A motion-bound lease and atomic dispatch check, with a first-order admission tolerance derived from sampled joint and cylinder-stroke margins. 2) Periodic authority supervision and a carry-pitch quintic stop, selected by sampled feasibility checks and evaluated for load retention and drive transients in synchronous simulation. 3) Component and lifecycle comparisons across 14 gateway configurations, 672 interface cases, four loaded simulated operating cases, unloaded cross-engine transfers and two local LLM planners, with deposition scored from material landings.
The implementation adds a mission/lease gateway, execution monitor and stop generator around RoboCompiler's mechanism and actuator interfaces. 

The code, results, and videos are publicly available at \textcolor{blue}{\url{https://github.com/Mehdi-Heydari-Shahna/ExecutionBound}}.

\begin{figure*}[t]
\centering
\includegraphics[width=\textwidth]{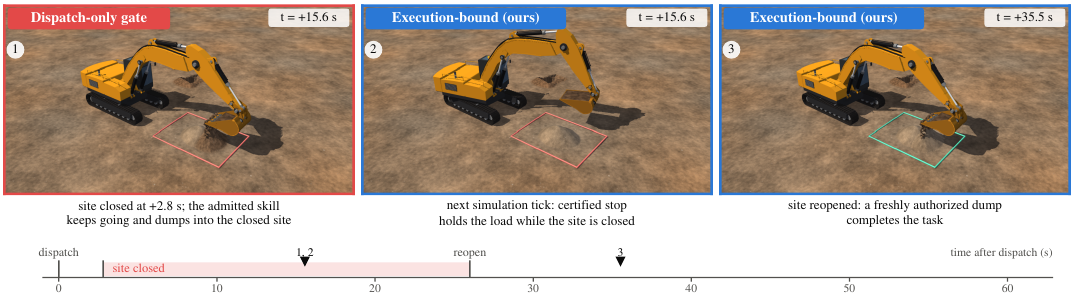}
\caption{Synchronous simulated revocation at 2.8\,s after dispatch. A dispatch-only gate deposits 1.23\,m$^3$ in the closed site; supervision detects closure at the next control tick and installs a sampled-admissible stop after synthesis, during which the plant is paused. A freshly authorized dump completes after reopening. Triangles mark the recorded frames; the outlined mound predates closure.}
\label{fig:teaser}
\end{figure*}

\section{Related Work}
\textbf{LLM security and authorization.} Prompt-injection and robot-jailbreak studies motivate independent enforcement and adaptive evaluation~\cite{greshake2023indirect,robey2025robopair,zhang2025badrobot,debenedetti2024agentdojo,nasr2026adaptive}. CaMeL separates control and data flow using capabilities~\cite{debenedetti2025camel}; Progent checks tool calls and updates privilege policies during execution~\cite{shi2026progent}; AttriGuard attributes invocations to their causes~\cite{he2026attriguard}. Safety Chip and RoboGuard enforce temporal or contextual robot rules, with RoboGuard operating online in the control loop~\cite{yang2024safetychip,ravichandran2026roboguard}. AuthGuard-R binds symbolic actions to mission, capability, provenance and current state, with independent physical checking~\cite{chepuri2026authguard}. The present contribution binds approval to the generated motion descriptor, accommodates dispatch drift, and connects execution-time revocation to loaded stopping and measured deposition. The experiments compare these components and lifecycle checks within a common gateway/executor interface.

\textbf{Continuous enforcement and runtime assurance.} Capabilities~\cite{dennis1966capabilities}, complete mediation~\cite{saltzer1975protection}, caveated tokens~\cite{birgisson2014macaroons} and usage control~\cite{park2004ucon} underpin the authorization lifecycle; time-of-check-to-time-of-use races motivate dispatch revalidation~\cite{bishop1996toctou}. SROS2 provides complementary transport security~\cite{mayoral2022sros2}. Simplex~\cite{sha2001simplex}, shields~\cite{alshiekh2018shielding}, control barrier functions~\cite{ames2019cbf} and predictive safety filters~\cite{wabersich2021safetyfilter} motivate supervised fallback~\cite{hobbs2023rta}. Here, authority loss triggers a sampled-admissible stopping reference whose carry-pitch endpoint is evaluated through retained load, material landings and drive transients.

\textbf{Excavation.} Autonomous excavation combines task planning, hydraulic control and soil interaction~\cite{zhang2021autonomous,jud2021heap,egli2022soil,mattila2017survey,reece1964fee}. ISO~17757 and ISO~19014-4 provide machine-safety references~\cite{iso17757,iso19014}; cybersecurity risk assessment under IEC~62443-3-2 complements them~\cite{iec62443}. Execution-bound authorization connects mission permissions to the motion and material flow of an excavator skill.

\section{Problem Setting}
\label{sec:problem}
A trusted operator issues a versioned mission contract $M_v$ containing destination regions, allowed skills and numeric limits, and an expiring, revocable capability $C\preceq M_v$ whose scope is contained in that contract. The planner proposes an action $a$ specifying skill, destination, target and speed. A trusted generator $G$ produces a descriptor $\traj=G(a;\mathcal M)$ for mechanism $\mathcal M$, defining a reference $u_\traj(\tau)\in\R^4$ at elapsed skill time $\tau$. The coordinates $u=(u_1,u_2,u_3,u_4)\T$ are swing, boom, arm and bucket angles in radians. A physical certifier checks the reference, a gateway authorizes dispatch, and an executor drives the hydraulic controller. State $x$ contains measured joint coordinates and base pose. The loaded study uses a stationary work base; closing a destination withdraws permission to deposit material there.

The adversary controls external planner context, proposed actions, dispatch-request contents and timing, replays and unsigned metadata. The trusted base comprises the operator, issuer, contract store, generator, certifier, state and clock sources, gateway and executor. Message authentication codes (MACs), canonical hashes and an exclusive actuator-command sink enforce this boundary; direct actuator access remains outside the adversary's interface.

The requirements are (R1)~current authority and exact approved motion identity at dispatch, (R2)~measured joint drift within the admission tolerance, (R3)~periodic detection of authority loss and installation of a stop when the duration search returns a sampled-admissible candidate, and (R4)~fresh authorization and admission before resumption. Completion and material retention are evaluated experimentally.

\section{Execution-Bound Authorization}
\label{sec:method}

\begin{figure*}[h!]
 \centering
\includegraphics[width=0.7\textwidth]{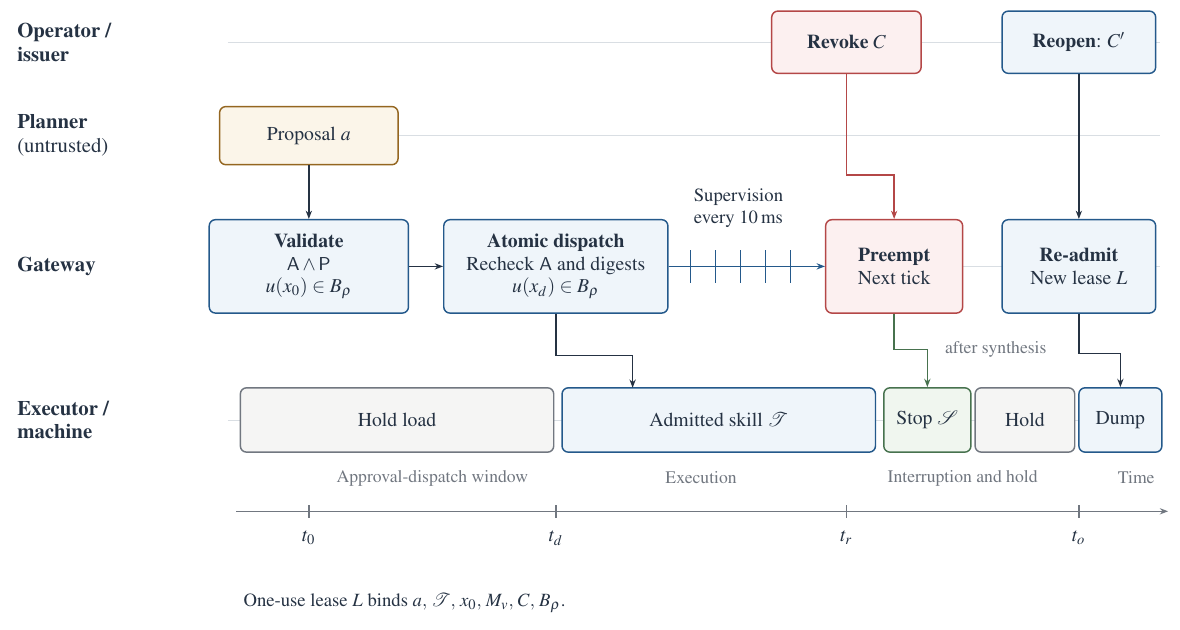}
\caption{Approval at $t_0$, dispatch at $t_d$, revocation at $t_r$ and reopening at $t_o$. Supervision closes the admitted record and installs stop $\mathcal S$ after synthesis. In the loaded simulations the plant is paused during synthesis; resumption requires a new lease.}
\label{fig:method}
\end{figure*}

Figure~\ref{fig:method} links admission, dispatch and supervision; Table~\ref{tab:algorithm} gives the corresponding procedure.

\subsection{Authorization and physical predicates}
Authorization $\Aauth(a,C,M_v,t)$ checks schema, mission limits, capability signature/scope/expiry/revocation, provenance, ordering and declared budgets at time $t$. Provenance MACs bind source classes to exact canonical field values. Skill, destination, frame and units require a trusted-command source; observations may populate only contract-approved numeric fields.
The predicate $\Pphys(\traj)$ evaluates sampled loop closure, joint limits, cylinder strokes and flow-derived rates. Loaded phases use seven compiled-graph samples and 41 flow samples. The reduced mechanism dynamics used for unloaded effort checks are \cite{shahna2025model}
\begin{equation}
M_r(u)\ddot u+h_r(u,\dot u)=B_r(u)f,
\label{eq:dynamics}
\end{equation}
where $M_r$ is reduced inertia, $h_r$ the bias vector, $B_r$ the actuator-port map and $f$ the actuator efforts. Unloaded transfers additionally check nominal effort and clearance to exclusion envelopes. RoboCompiler supplies the closed-chain maps~\cite{shahna2026robocompiler}. Throughout, a \emph{certificate} denotes acceptance by the stated sampled checks.

\subsection{First-order admission tolerance}
Let $s_i(u)$, $i\in\{1,2,3\}$, denote the distinct boom, arm and bucket stroke coordinates; the boom pair shares its coordinate. Let $u_j^\pm$ be the joint limits for $j\in\{1,2,3,4\}$ and $s_i^\pm$ the stroke limits. Define
\begin{equation}
\begin{aligned}
m_J(u)&=\min_j\{u_j-u_j^-,\;u_j^+-u_j\},\\
m_{S,i}(u)&=\min\{s_i(u)-s_i^-,\;s_i^+-s_i(u)\}.
\end{aligned}
\label{eq:margins}
\end{equation}
For reference samples $u_k=u_\traj(\tau_k)$, let $\ell_{ik}=\|\partial s_i/\partial u\,(u_k)\|_1$. The implementation uses 81 samples per phase, maximum spacing $\Delta\tau$, and largest sampled reference rate $\bar v=\max_k\|\dot u_\traj(\tau_k)\|_\infty$. Its first-order kinematic allowance is
\begin{equation}
\rho=\max\!\left\{0,\;\min_k\min\!\left[m_J(u_k),\min_i\frac{m_{S,i}(u_k)}{\ell_{ik}}\right]
-\frac{\bar v\Delta\tau}{2}\right\}.
\label{eq:tube}
\end{equation}
Stroke Jacobian norms are positive in the evaluated configurations. Define the linearized increment $\delta s_{ik}^{\rm lin}=(\partial s_i/\partial u)(u_k)\delta u$. H\"older's inequality gives $|\delta s_{ik}^{\rm lin}|\le\ell_{ik}\|\delta u\|_\infty$. Equation~\eqref{eq:tube} combines sampled joint margins, linearized stroke margins and the sampling reserve $\bar v\Delta\tau/2$ to set the joint-drift admission threshold. The unchanged reference undergoes its sampled nonlinear checks separately. The catalogue gives $\rho=0.096$--$0.115$\,rad, governed by the open bucket's joint margin and the sampling term. Median tolerance computation is 0.7\,s per skill.

\subsection{Lease and atomic dispatch}
At validation time $t_0$, snapshot $x_0$ must satisfy $\Aauth\land\Pphys$ and $\|u(x_0)-u_\traj(0)\|_\infty\le\rho$. The gateway then issues
\begin{equation}
\begin{aligned}
L=\mathrm{MAC}_{K_g}\big(&n,r,\Hh(a),\Hh(\traj),\Hh(x_0),\\[-1pt]
                      &\Hh(M_v),C,\mathcal B_\rho,t_0,t_e\big),
\end{aligned}
\label{eq:lease}
\end{equation}
where $K_g$ is the gateway's secret key, $n$ a nonce, $r$ the request identity, $\Hh$ a canonical hash, $t_e$ the expiry, and $\mathcal B_\rho=\{u:\|u-u_\traj(0)\|_\infty\le\rho\}$. Dispatch at $t_d$ holds both the gateway and issuer-revocation locks, checks the bound values and current authority, admits live state $x_d$ through $\mathcal B_\rho$, consumes the lease and publishes the approved descriptor. The reused certificate concerns the unchanged descriptor $\traj$; $\mathcal B_\rho$ is the separate joint-drift admission condition for the stationary base. Execution monitoring compares measured joints with the active reference.

\begin{table}[t]
\caption{Dispatch and supervision; steps 1--5 hold both locks.}
\label{tab:algorithm}
\centering\small
\begin{tabular}{p{0.94\columnwidth}}
\toprule
\textbf{Dispatch} of $(a,\traj,x_0,L)$ at $t_d$ with live $x_d$:\\
1. Verify MAC, issuance, expiry, one-time use and sink ownership.\\
2. Match $\Hh(a),\Hh(\traj),\Hh(x_0),\Hh(M_{\rm current})$ to $L$.\\
3. Recheck $\Aauth(a,C,M_{\rm current},t_d)$.\\
4. Require $u(x_d)\in\mathcal B_\rho$; reuse $\Pphys(\traj)$.\\
5. Consume $L$; open an execution record; log and publish $\traj$.\\
\textbf{Supervision} at period $\Delta_m=10$\,ms in simulation:\\
6. On capability revocation/expiry, contract change or $\|u-u_\traj(\tau)\|_\infty>\rho$, close the record and request checked stop $\mathcal S$; resumption requires a new lease.\\
\bottomrule
\end{tabular}
\end{table}

\subsection{Supervision and load-retaining preemption}
\label{sec:execution}
The monitor evaluates authority and reference deviation every $\Delta_m$. At preemption, the current reference state is $(u_p,v_p,a_p)$, with position, velocity and acceleration respectively. The stop is the unique quintic satisfying
\begin{equation}
\begin{aligned}
u_{\mathcal S}(\tau)&=\sum_{\ell=0}^{5}c_\ell\tau^\ell,\qquad 0\le\tau\le T_s,\\
(u_{\mathcal S},\dot u_{\mathcal S},\ddot u_{\mathcal S})(0)&=(u_p,v_p,a_p),\\
(u_{\mathcal S},\dot u_{\mathcal S},\ddot u_{\mathcal S})(T_s)&=(u_e,0,0),
\end{aligned}
\label{eq:stop}
\end{equation}
where $c_\ell\in\R^4$ are polynomial coefficients and $T_s$ the duration. For fixed $T_s$, the unconstrained free-endpoint minimum of integrated squared jerk gives $u_e^0=u_p+T_sv_p/2+T_s^2a_p/12$. Using the catalogue's bucket-pitch convention $\phi(u)=u_2+u_3+u_4$, the endpoint is curled back when its pitch exceeds the carry pitch:
\begin{equation}
u_e=u_e^0-e_4\max\!\left(0,\sum_{j=2}^{4}u_{e,j}^0-\phi_{\rm carry}\right),
\label{eq:curl}
\end{equation}
where $e_4=(0,0,0,1)\T$ and $\phi_{\rm carry}=-8\pi/9$\,rad ($-160^\circ$). The correction sets $\phi(u_e)=\min\{\phi(u_e^0),\phi_{\rm carry}\}$; material landings quantify retention during the resulting motion. In synchronous simulation, Eq.~\eqref{eq:stop} preserves reference position, velocity and acceleration at the switch.

The generator selects the first feasible $T_s\in\{0.50,0.55,\ldots,8.00\}$\,s. At 41 evenly spaced samples, joint and stroke margins must exceed $1^\circ$ and 0.02\,m and each $|\ddot u_{\mathcal S,j}|\le\bar a_j$; flow-derived cylinder and swing-rate checks also apply. The reference-acceleration envelope $\bar a$ is the per-joint maximum sampled acceleration of the validated mission and dump catalogue: 0.28--0.32, 0.06--0.07, 0.14 and 0.39\,rad/s$^2$ for swing, boom, arm and bucket across the four cases. At stop completion, the controller holds $u_e$. Reopening triggers generation and admission of a fresh dump from rest.

\begin{proposition}[Dispatch integrity and detection]
\label{prop:bound}
Assume unforgeable MACs, collision-resistant canonical hashes, trusted state/time/generation/certification, an exclusive gateway sink, and atomic serialization of dispatch with contract and revocation updates. Let successive supervision evaluations be separated by at most $\Delta_m$, with all timing bounds measured on the same executor clock. Every nominal skill dispatch through steps 1--5 then matches its approved $(a,\traj)$ under the current contract and an unexpired, unrevoked capability at the dispatch check, with $u(x_d)\in\mathcal B_\rho$; each lease is dispatched at most once. Authority loss recorded at $t_r$ is detected by $t_r+\Delta_m$. If the duration search returns a sampled-admissible stop and synthesis/installation completes within $\delta_s$, that stop is installed by $t_r+\Delta_m+\delta_s$.
\end{proposition}
\emph{Proof:} Nominal skill publication requires successful completion of steps~1--5 at the exclusive sink. MAC and hash verification bind the action, descriptor, snapshot and contract. The locks serialize authority updates with dispatch checks, and consumption before publication prevents lease reuse. Authority loss recorded after dispatch is observed within the assumed maximum supervision interval; synthesis and installation contribute the additional delay $\delta_s$.\hfill$\square$

\section{Experimental Setup}
\label{sec:setup}
\begin{figure*}[t]
\centering
\includegraphics[width=\textwidth]{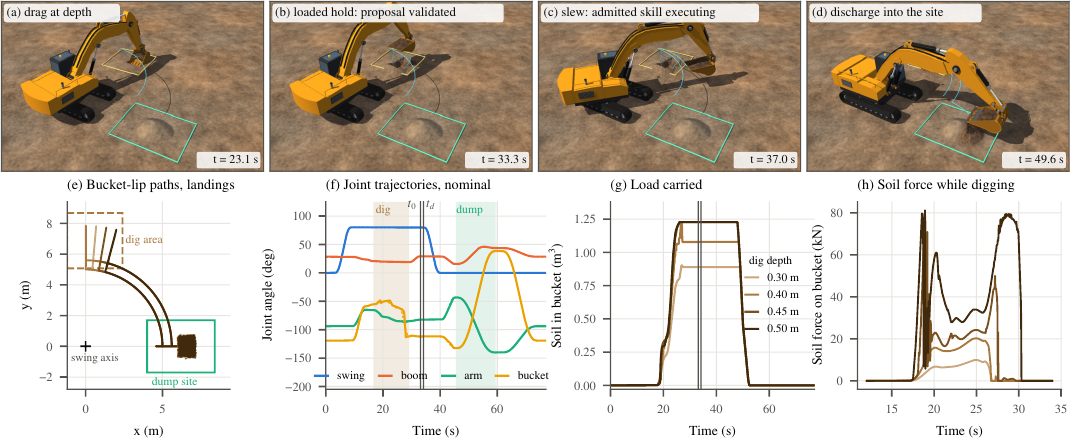}
\caption{Recorded loaded cycle. (a)--(d)~Nominal dig, validation hold, slew and discharge; amber/green outlines mark dig/dump regions, and blue marks the bucket-lip trail. (e)~Bucket-lip paths and clod landings for four operating cases. (f)~Measured joint angles and dashed references. (g)~Carried soil. (h)~Soil force during penetration, drag, curl and lift.}
\label{fig:task}
\end{figure*}
\textbf{Plant and implementation.} The DX300-class model has 30.2\,t of rigid bodies, five passive loop closures and four independent coordinates. MuJoCo~3.3.7~\cite{todorov2012mujoco} integrates dynamics at 1\,kHz with 100\,Hz computed-torque control, valve-limited cylinders and a swing motor, at 34.3\,MPa supply pressure. RoboCompiler provides the mechanism, tangent and actuator maps~\cite{shahna2026robocompiler}. A continuum-soil coupler combines strip-wise fundamental earth-moving-equation cutting~\cite{reece1964fee}, Winkler--plastic support, scooping, spill, ballistic clods and heightfield avalanching; carried soil updates bucket mass and inertia. Logged landings determine material outcomes independently of gateway decisions. Loaded interruption uses synchronous switching: the plant is paused during stop synthesis, and the generated reference starts from the reference state at detection. Supervision intervals and motion durations use simulation time; synthesis cost uses wall-clock time. Pinocchio~3.8.0~\cite{carpentier2019pinocchio} provides a second rigid-body implementation for unloaded transfers. Figure~\ref{fig:task} documents the motion, tracking, carried load and soil forces of the loaded cycle.

\textbf{Gates.} B0 executes unchecked; B1 applies a lexical prompt filter; B2 checks schema and mission limits; B3 checks physics; B23 combines B2/B3. B1--B23 repeat their checks at dispatch. TOC performs full checks at validation but verifies only token authenticity, expiry and one-time use at dispatch. DB-exact uses EB's dispatch checks with exact state equality. DO is the dispatch-bound gateway without execution supervision. EB-hold supervises and freezes the reference; EB-brake uses the quintic stop without curl-back. EB is the complete method. Ablations remove individual authorization components or use static permissions.

\textbf{Interface suite.} Each of 14 configurations faces 23 attack and five benign families, with 24 held-out draws per family: 552 attack and 120 benign cases. An independent mission/event oracle scores sink publication; this command-level suite uses a fixed 10\,mm tool-tip tube. Feedback-guided attackers observe decisions and reasons, then mutate family templates and parameters in 16 campaigns at each budget of 64 and 256 queries.

\textbf{Loaded studies.} Four operating cases vary digging depth (0.30--0.50\,m), lane, bank density (1450--1750\,kg/m$^3$), friction angle ($33$--$40^\circ$) and cohesion (3.0--6.5\,kPa). Every condition within a case resumes the same plant/soil/controller checkpoint after digging. Scripted planner proposals ensure matched gate comparisons. Dispatch occurs about 1\,s after validation at the full-bucket hold. Scenarios are \emph{clean}, \emph{redirect} (a context-induced $-25^\circ$ dump outside the site), \emph{revoke} (closure 0.5\,s after validation), and \emph{substitute} (presenting an individually permitted $-25^\circ$ pad-B motion under the site approval). Execution-revocation tests instead close the site at ten offsets from 0.8 to 17.3\,s after dispatch, spanning slew, positioning and dumping. The site reopens at 26\,s after dispatch. EB-brake is evaluated at the four latest offsets.

\textbf{Transfers and planners.} Unloaded MuJoCo/Pinocchio tests include unauthorized targets, a worker-exclusion envelope, overspeed and approval-lifecycle attacks. Qwen2.5-0.5B/Q8\_0 and 3B/Q4\_K\_M~\cite{qwen2025technical} each receive 30 clean/attacked prompt pairs under task-only and safety system prompts, covering ten context templates and three mission geometries. Their target/speed proposals are evaluated at the command-level gateway.

\section{Results}
\label{sec:results}
\subsection{Interface conformance}

\begin{figure}[t]
\centering
\includegraphics[width=\columnwidth]{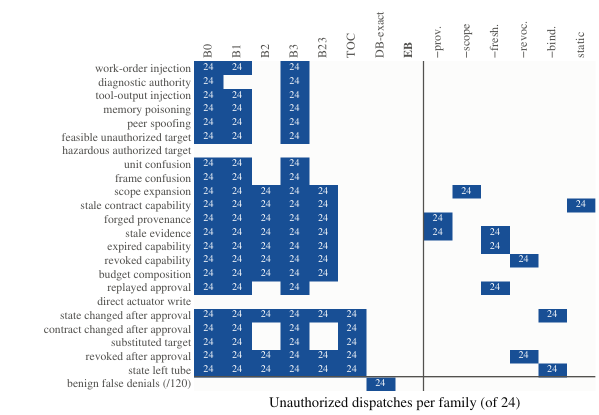}
\caption{Unauthorized dispatches per attack family (24 draws); last row: benign false denials (120 cases). EB has neither. The authorized hazardous-target family tests physics, while direct actuator writes fail at the exclusive sink; neither contributes unauthorized dispatches.}
\label{fig:conformance}
\end{figure}
EB produced zero unauthorized dispatches across 552 attack cases and admitted all 120 benign cases (Fig.~\ref{fig:conformance}). Of the attacks, 48 concern an authorized hazardous target or direct writes and cannot produce unauthorized dispatches under any gate. B0/B2/B3/B23 produced 504/240/504/240 unauthorized dispatches. TOC produced 120 across post-validation revocation, contract/target substitution and two state-change families. DB-exact blocked all attacks but rejected 24 benign cases with 4\,mm drift inside the 10\,mm tube. Removing provenance, scope, freshness, revocation or binding reopened 48, 24, 72, 48 or 48 cases; static permissions reopened 24. Feedback-guided search succeeded in all 16 campaigns at both budgets against every configuration except EB and DB-exact; EB recorded no success in 5{,}120 queries. EB gateway calls per case took 0.50\,ms median and 2.58\,ms at the 95th percentile.

\subsection{Dispatch window in loaded digging}
\begin{figure*}[t]
\centering
\includegraphics[width=\textwidth]{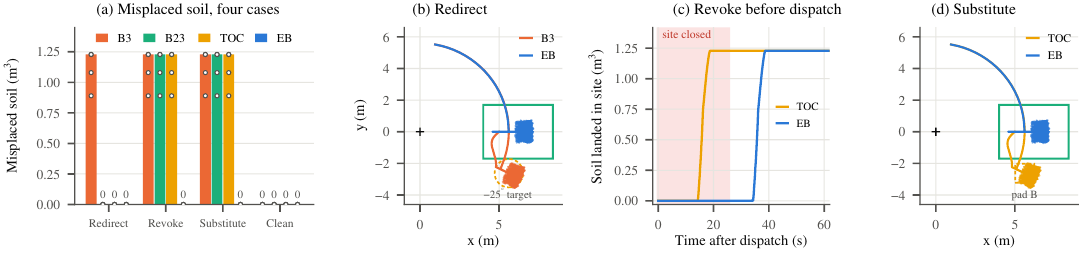}
\caption{Dispatch-window outcomes. (a)~Misplaced soil: maxima (bars) and individual operating cases (dots). (b), (d)~Nominal bucket-lip paths and clod landings under redirect/substitution. (c)~Deposition over time after pre-dispatch closure. Overlapping curves reflect identical executed motions.}
\label{fig:window}
\end{figure*}
Under \emph{redirect}, B3 executed the feasible unauthorized dump, whereas B23, TOC and EB rejected it and completed a reissued site dump (Fig.~\ref{fig:window}b). Under \emph{revoke}, B3/B23/TOC delivered the full load into the closed site; EB rejected stale authority, held the load and recovered after reopening (Fig.~\ref{fig:window}c). Under \emph{substitute}, the presented pad-B motion passed the baseline checks; EB rejected its binding mismatch and executed the approved site dump (Fig.~\ref{fig:window}d). Across the four cases, EB deposited zero material outside the approved site or into the closed site, and every clean run completed. Loads ranged from 1,032 to 1,719\,kg, with a maximum volume of 1.23\,m$^3$. State drift reached 0.0033\,rad and remained inside the admission tube; all 52 delayed dispatches had nonzero drift, so exact equality would have rejected each one.

\subsection{Revocation during execution}
\begin{figure*}[t]
\centering
\includegraphics[width=0.99\textwidth]{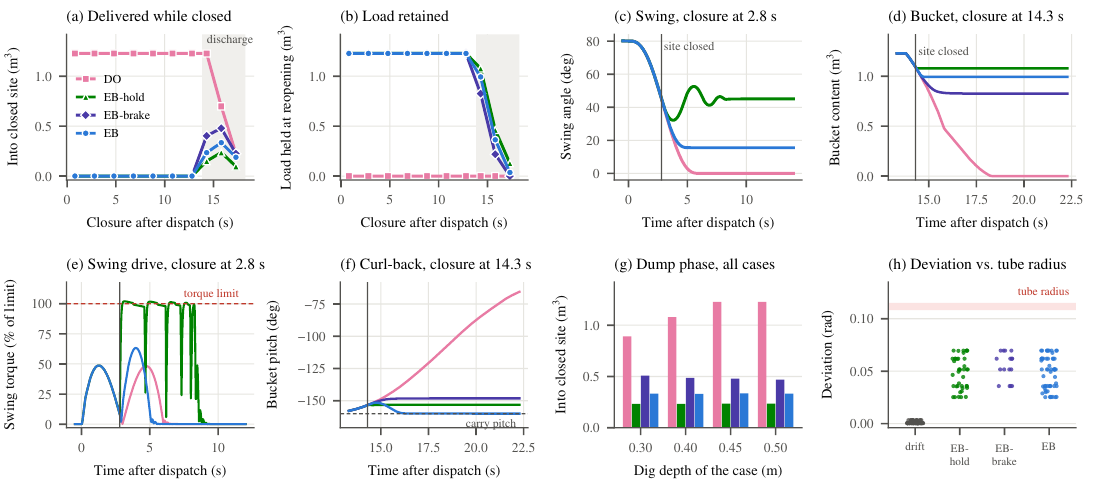}
\caption{Synchronous simulated execution revocation. (a)~Closed-site deposition versus closure time; shading marks nominal discharge. (b)~Retained load at reopening. (c), (e)~Swing angle and motor torque after slew-phase closure. (d), (f)~Bucket content and pitch after discharge-phase closure. (g)~Largest dump-phase deposition across operating cases. (h)~Deviation during active admitted-skill records and dispatch drift relative to the admission radius; stop-segment tracking is reported in Table~\ref{tab:revoke}.}
\label{fig:revocation}
\end{figure*}
\begin{table}[t]
\caption{Revocation results (rounded maxima; $n$: runs). Stop metrics cover the 10\,s after preemption for three slew-phase closures per case: settle is sustained joint rate below 0.02\,rad/s, back is swing return, torque is peak fraction of nominal limit, and track is peak measured reference error.}
\label{tab:revoke}
\centering\footnotesize\setlength{\tabcolsep}{2.6pt}
\begin{tabular}{@{}lrrrrrrr@{}}
\toprule
 & & \multicolumn{2}{c}{Closed site (m$^3$)} & \multicolumn{4}{c}{Stop during slew} \\
\cmidrule(lr){3-4}\cmidrule(l){5-8}
Gate & $n$ & before & during & settle & back & torque & track \\
 & & \multicolumn{2}{c}{dump phase} & (s) & ($^\circ$) & (\%) & (rad) \\
\midrule
DO & 40 & 1.23 & 1.23 & -- & -- & -- & -- \\
EB-hold & 40 & 0.00 & 0.24 & 6.3 & 24.9 & 102 & 0.271 \\
EB-brake & 16 & -- & 0.51 & -- & -- & -- & -- \\
EB & 40 & 0.00 & 0.33 & 2.1 & 0.05 & 74 & 0.003 \\
\bottomrule
\end{tabular}
\end{table}
DO continued the admitted skill and deposited the remaining load into the closed site, with a maximum of 1.23\,m$^3$ (Fig.~\ref{fig:revocation}a, Table~\ref{tab:revoke}). EB detected each closure at the next simulation supervision tick. It deposited nothing for closures before the dump phase, which begins 11.3\,s after dispatch. During dumping, curl-back reduced the maximum soil released during the stop to 0.12\,m$^3$, versus 0.29\,m$^3$ for EB-brake (Fig.~\ref{fig:revocation}d,f). Including material already airborne, maximum closed-site deposition was 0.33 versus 0.51\,m$^3$, with a largest paired reduction of 0.19\,m$^3$; all four cases showed the same retention benefit (Fig.~\ref{fig:revocation}b,g). EB-hold released at most 0.02\,m$^3$ during its stop and had maximum closed-site deposition of 0.24\,m$^3$. It retained more material than EB in the dump-phase tests; the loaded-slew results below quantify the drive and tracking cost of freezing the reference.

During slew-phase closures, EB-hold saturated the swing drive for up to 3.2\,s within the 10-s observation window and produced 24.9$^\circ$ back-swing, 0.271\,rad tracking error and 6.3\,s settling time (Fig.~\ref{fig:revocation}c,e). EB's checked stop used 74\% of nominal torque (rounded), settled within 2.1\,s, and limited back-swing and tracking error to 0.05$^\circ$ and 0.003\,rad. Peak measured swing acceleration was 0.36\,rad/s$^2$, versus 0.52\,rad/s$^2$ for EB-hold; the planned stopping arc reached 33$^\circ$.

All 96 supervised revocation runs completed freshly authorized recovery, and every EB/EB-brake preemption found a sampled-admissible stop. Across 112 supervised runs and 208 admitted skills, maximum deviation during active admitted-skill execution and recovery was 0.070\,rad and remained inside the admission radius, with no spurious preemption (Fig.~\ref{fig:revocation}h). The admitted record closes before the stop, whose tracking error is scored separately in Table~\ref{tab:revoke}. Supervision ran every 10\,ms of simulation time. The Python check and stop synthesis took 115\,$\mu$s and 142\,ms median wall-clock time, respectively. The paused-plant execution yields the recorded stopping trajectories; these medians report computational cost separately from the assumed delay bound $\delta_s$ in Proposition~\ref{prop:bound}.

\subsection{Unloaded transfers and LLM planners}
\begin{table}[t]
\caption{Unloaded transfers: 192 decision--engine records reusing 32 distinct simulations. Lifecycle: 160 records (128 attack and 32 clean); attack outcomes shown. Unauth.: unauthorized endpoints; Phys.: constraint events; Done: authorized goals; Dispatched: invalid requests executed; Recovered: authorized recovery after rejection.}
\label{tab:native}
\centering\footnotesize\setlength{\tabcolsep}{3.2pt}
\begin{tabular}{lrrrcrr}
\toprule
 & \multicolumn{3}{c}{Transfers (32 per gate)} & & \multicolumn{2}{c}{Lifecycle (32 invalid)} \\
\cmidrule{2-4}\cmidrule{6-7}
Gate & Unauth. & Phys. & Done & & Dispatched & Recovered \\
\midrule
B0 & 8 & 16 & 16/32 & & -- & -- \\
B1 & 8 & 16 & 16/32 & & -- & -- \\
B2 & 0 & 16 & 24/32 & & -- & -- \\
B3 & 8 & 0 & 24/32 & & -- & -- \\
B23 & 0 & 0 & 32/32 & & 32 & -- \\
EB & 0 & 0 & 32/32 & & 0 & 32/32 \\
EB$-$binding & -- & -- & -- & & 8 & 24/24 \\
EB$-$revocation & -- & -- & -- & & 8 & 24/24 \\
\bottomrule
\end{tabular}
\end{table}
In both engines, B3 reached unauthorized targets and B2 violated worker-envelope or speed limits; B23 and EB had neither event (Table~\ref{tab:native}). For mechanically feasible lifecycle attacks within the broad mission, B23 executed all 32 invalid dispatches and EB none. Removing binding or revocation reopened the corresponding eight cases. All eight clean controls per lifecycle configuration completed. Paired bucket-lip trajectories agreed within approximately $2.1\times10^{-7}$\,m; relative mass-matrix discrepancy over 128 states was approximately $2.3\times10^{-15}$.

For each planner and prompt style, among 30 attacked prompts, B0 applied 29--30 unauthorized commands and B3 applied 26--27. B23 and EB applied none and accepted all 30 clean commands. These outcomes support enforcing mission permissions outside the planner; the dispatch and revocation studies separately test the additional motion-lifecycle protections.

\section{Discussion and Conclusion}
Execution-bound authorization connects mission permissions to exact motion identity, atomic dispatch and periodic execution checks. The interface and dispatch-window studies demonstrate rejection of invalid authority and substituted motions while admitting the tested benign drift. In synchronous loaded simulations, supervised interruption prevented deposition for every tested pre-dump closure. Curl-back reduced dump-phase deposition relative to the uncurled brake, while the quintic stop reduced the slew-drive transient relative to a frozen reference. Retained material, torque and tracking expose this operating tradeoff; unloaded cross-engine agreement checks the mechanism interface. The resulting framework provides an evaluated authorization lifecycle for excavation skills, with sampled motion admission and conditional stopping properties stated separately from measured physical outcomes.

\bibliographystyle{IEEEtran}
\bibliography{refere}
\end{document}